\documentclass[11pt,a4paper]{article}
\usepackage{hyperref}
\usepackage[hyperref]{acl2018}
\usepackage{times}
\usepackage{latexsym}
\usepackage{ifthen}
\usepackage{url}
\usepackage{amsmath}
\usepackage{amsthm}
\usepackage{amsfonts}
\usepackage{amssymb}
\usepackage{bbm}
\usepackage{graphicx}
\usepackage{subcaption}
\usepackage{booktabs}
\usepackage{soul}

\usepackage{color}
\usepackage{colortbl}
\usepackage{algorithm}
\usepackage{algorithmic}
\usepackage{amsthm}

    \newtheorem{lemma}{Lemma}[section]
    \newtheorem{proposition}{Proposition}[section]
    \newtheorem{theorem}{Theorem}[section]
    
    \newtheoremstyle{TheoremNum}
        {\topsep}{\topsep}              
        {\itshape}                      
        {}                              
        {\bfseries}                     
        {.}                             
        { }                             
        {\thmname{#1}\thmnote{ \bfseries #3}}
    \theoremstyle{TheoremNum}
    \newtheorem{thm}{Theorem}

    \newtheorem{prop}{Proposition}

\aclfinalcopy{} 
\def\aclpaperid{1604} 

\title{The price of debiasing automatic metrics in natural language evaluation}

\author{%
Arun Tejasvi Chaganty$^*$ \and
Stephen Mussmann\thanks{\ Authors contributed equally.}  \and
Percy Liang \\
  Computer Science Department,
  Stanford University \\
  {\tt \{chaganty,mussmann,pliang\}@cs.stanford.edu}
}
\date{}

\providecommand\st{\ensuremath{\mathcal{t}}}

\providecommand\sX{\ensuremath{\mathcal{X}}}

\providecommand\sZ{\ensuremath{\mathcal{Z}}}

\newcommand\inv[1]{\ensuremath{\frac{1}{#1}}}

\newcommand\eqdef{\ensuremath{\stackrel{\rm def}{=}}} 
\newcommand\refeqn[1]{(\ref{eqn:#1})}
\newcommand\refeqns[2]{(\ref{eqn:#1}) and (\ref{eqn:#2})}

\newcommand\refsec[1]{Section~\ref{sec:#1}}

\newcommand\reffig[1]{Figure~\ref{fig:#1}}

\newcommand\reftab[1]{Table~\ref{tab:#1}}
\newcommand\refapp[1]{Appendix~\ref{sec:#1}}
\newcommand\refthm[1]{Theorem~\ref{thm:#1}}

\newcommand\reflem[1]{Lemma~\ref{lem:#1}}

\newcommand\refalg[1]{Algorithm~\ref{alg:#1}}

\ifthenelse{\isundefined{\definition}}{\newtheorem{definition}{Definition}}{}
\ifthenelse{\isundefined{\assumption}}{\newtheorem{assumption}{Assumption}}{}
\ifthenelse{\isundefined{\hypothesis}}{\newtheorem{hypothesis}{Hypothesis}}{}
\ifthenelse{\isundefined{\proposition}}{\newtheorem{proposition}{Proposition}}{}
\ifthenelse{\isundefined{\theorem}}{\newtheorem{theorem}{Theorem}}{}
\ifthenelse{\isundefined{\lemma}}{\newtheorem{lemma}{Lemma}}{}
\ifthenelse{\isundefined{\corollary}}{\newtheorem{corollary}{Corollary}}{}
\ifthenelse{\isundefined{\alg}}{\newtheorem{alg}{Algorithm}}{}
\ifthenelse{\isundefined{\example}}{\newtheorem{example}{Example}}{}
\newcommand{\E}{\ensuremath{\mathbb{E}}} 

\newcommand{\Var}{\operatorname{Var}}
\newcommand{\Cov}{\operatorname{Cov}}

\newcommand\DE{\text{DE}}
\newcommand\musimple{\hat \mu_\text{mean}}
\newcommand\mucontrol{\hat \mu_\text{cv}}

\begin{document}
\maketitle

\begin{abstract}
For evaluating generation systems, automatic metrics such as BLEU cost nothing to run but have been shown to correlate poorly with human judgment, leading to systematic bias against certain model improvements.
On the other hand, averaging human judgments, the unbiased gold standard, is often too expensive.
In this paper, we use control variates to combine automatic metrics with human evaluation to
obtain an unbiased estimator with lower cost than human evaluation alone.
In practice, however, we obtain only a 7--13\% cost reduction on evaluating summarization and open-response question answering systems.
We then prove that our estimator is optimal: there is no unbiased estimator with lower cost.
Our theory further highlights the two fundamental bottlenecks---the automatic
metric and the prompt shown to human evaluators---both of which need to be improved to obtain greater cost savings.

\end{abstract}

\section{\label{sec:intro}Introduction}

In recent years, there has been an increasing interest in tasks that require generating natural language, including
  abstractive summarization~\citep{nallapati2016abstractive},
  open-response question answering~\citep{nguyen2016ms,kovcisky2017narrativeqa}, 
  image captioning~\citep{lin2014microsoft},
  and open-domain dialogue~\citep{lowe2017ubuntu}.
Unfortunately, the evaluation of these systems remains a thorny issue because of the diversity of possible correct responses.
As the gold standard of performing human evaluation is often too expensive,
there has been a large effort developing
automatic metrics such as BLEU~\citep{papineni02bleu}, ROUGE~\citep{lin2004rouge}, METEOR~\citep{lavie2009meteor,denkowski2014meteor} and CiDER~\citep{vedantam2015cider}.
However, these have shown to be biased, correlating poorly with human metrics across different datasets and systems~\citep{liu2016evaluate,novikova2017why}.

\begin{figure*}[ht]
  \begin{subfigure}{0.49\textwidth}
    \includegraphics[width=\textwidth]{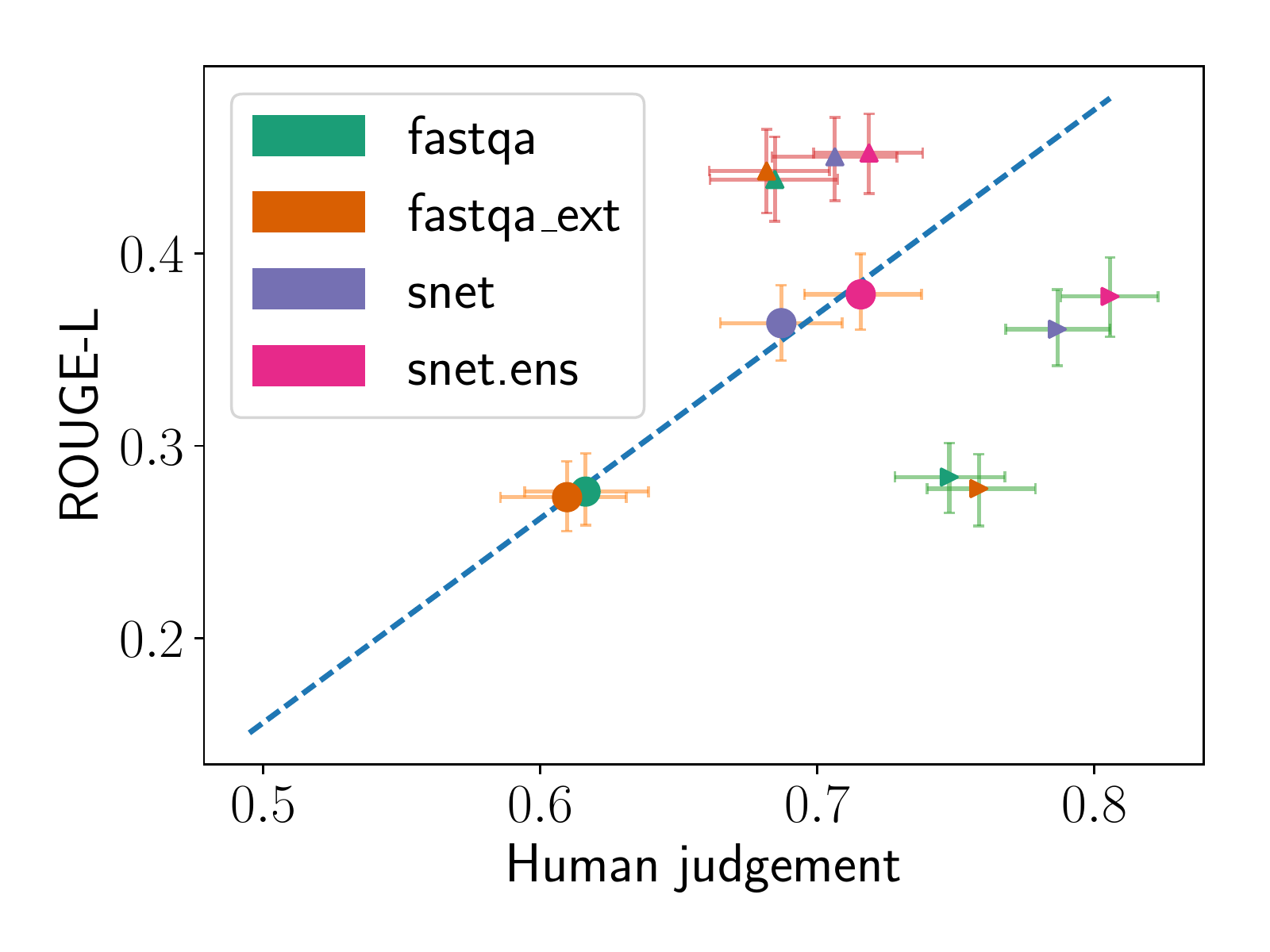}
    \caption{\label{fig:bias-msmarco-system} System-level correlation on the MS MARCO task}
  \end{subfigure}
  \hfill
  \begin{subfigure}{0.49\textwidth}
    \includegraphics[width=\textwidth]{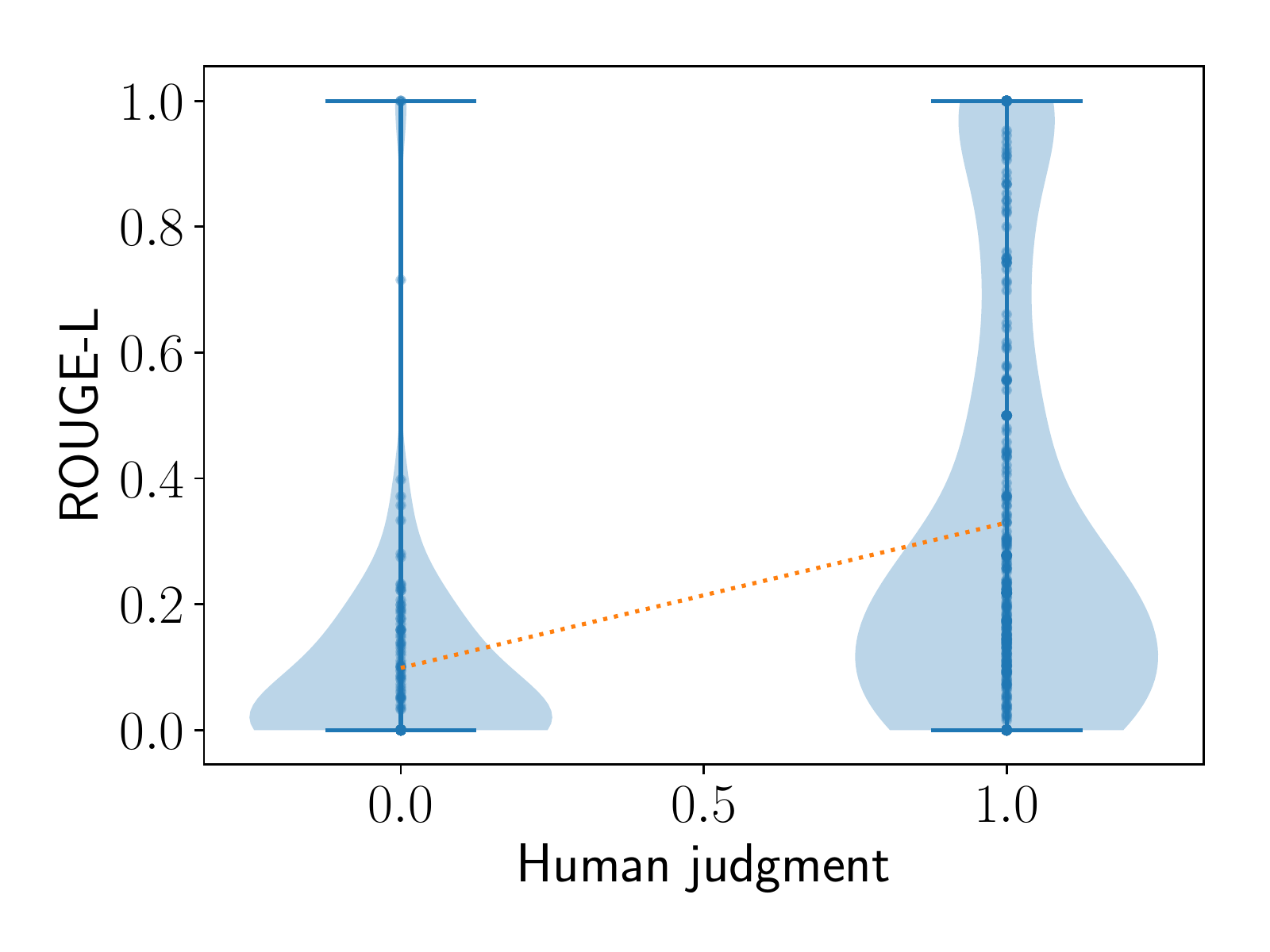}
    \caption{\label{fig:bias-msmarco-instance} Instance-level correlation for the \texttt{fastqa} system}
  \end{subfigure}
  \caption{\label{fig:bias-msmarco}
  (a) At a system-level, automatic metrics (ROUGE-L) and human judgment correlate well, but (b) the instance-level correlation plot
  (where each point is a system prediction) shows that the instance-level correlation is quite low ($\rho = 0.31$).
  As a consequence, if we try to locally improve systems to produce better answers ($\triangleright$ in (a)),
  they do not significantly improve ROUGE scores and vice versa ($\vartriangle$).
  }
\end{figure*}

Can we combine automatic metrics and human evaluation to obtain
an \emph{unbiased} estimate at \emph{lower cost} than human evaluation alone?
In this paper,
we propose a simple estimator based on control variates~\citep{ripley2009stochastic},
where we average differences between human judgments and automatic metrics
rather than averaging the human judgments alone.
Provided the two are correlated,
our estimator will have lower variance and thus reduce cost.

We prove that our estimator is \emph{optimal} in the sense
that no unbiased estimator using the same automatic metric can have lower variance.
We also analyze its data efficiency (equivalently, cost savings)---the factor reduction in number of human judgments needed to obtain the same accuracy versus naive human evaluation---and show that it depends solely on
two factors:
  (a) the annotator variance (which is a function of the human evaluation prompt) and
  (b) the correlation between human judgments and the automatic metric.
This factorization allows us to calculate typical and best-case data efficiencies and accordingly refine the evaluation prompt or automatic metric.

Finally, we evaluate our estimator on state-of-the-art systems from two tasks, summarization on the CNN/Daily Mail dataset~\cite{hermann2015read,nallapati2016abstractive}
and open-response question answering on the MS MARCOv1.0 dataset~\citep{nguyen2016ms}.
To study our estimators offline,
we preemptively collected 10,000 human judgments which cover several tasks and systems.\footnote{%
An anonymized version of this data and the annotation interfaces used can be found at \url{https://bit.ly/price-of-debiasing}.}
As predicted by the theory, we find that the data efficiency depends not only on the correlation between the human and automatic metrics,
but also on the evaluation prompt.
If the automatic metric had perfect correlation, our data efficiency would be around 3, while
  if we had noiseless human judgments, our data efficiency would be about 1.5.
In reality, the reduction in cost we obtained was only about 10\%,
suggesting that improvements in both automatic metric and evaluation prompt are needed.
As one case study in improving the latter, we show that, when compared to a Likert survey, measuring the amount of post-editing needed to fix a generated sentence
reduced the annotator variance by three-fold.

\section{\label{sec:bias} Bias in automatic evaluation}

\begin{table*}
  \begin{subtable}{\textwidth}
  \centering
\begin{tabular}{p{0.47\textwidth} p{0.47\textwidth}}
\toprule
\textbf{Question and reference answer} & \textbf{System answer} \textbf{(System; \texttt{Corr} / ROUGE-L)} \\
\midrule
\multicolumn{2}{l}{\textit{Examples where system is correct and ROUGE-L $> 0.5$ ($19.6\%$ or 285 of 1455 unique responses)}} \\
\midrule
 \small \textbf{Q.} what is anti-mullerian hormone \par
  \textbf{A.} Anti-Mullerian Hormone (AMH) is a protein hormone produced by granulosa cells (cells lining the egg sacs or follicles) within the ovary.
  &
\small it is a protein hormone produced by granulosa cells (cells lining the egg sacs or follicles) within the ovary. 
(\texttt{snet.ens}; $\checkmark$ / $0.86$)
\\
\midrule
\multicolumn{2}{l}{\textit{Examples where system is incorrect and ROUGE-L $> 0.5$ ($1.3\%$ or 19 of 1455 unique responses)}} \\
\midrule
\small   \textbf{Q.} at what gestational age can you feel a fetus move \par
  \textbf{A.} 37 to 41 weeks \textit{(incorrect reference answer)} &
\small 37 to 41 weeks
(\texttt{fastqa, fastqa.ext}; $\times$ / 1.0)
\\
\midrule
\multicolumn{2}{l}{\textit{Examples where system is correct and ROUGE-L $< 0.5$ ($56.0\%$ or 815 of 1455 unique responses)}} \\
\midrule
\small   \textbf{Q.} what is the definition of onomatopoeia \par
  \textbf{A.} It is defined as a word, which imitates the natural sounds of a thing. &
\small   the naming of a thing or action by a vocal imitation of the sound associated with it (as buzz, hiss).
  (\texttt{fastqa}; $\checkmark$ / $0.23$)
  \\
\midrule
\multicolumn{2}{l}{\textit{Examples where system is incorrect and ROUGE-L $< 0.5$ ($23.1\%$ or 336 of 1455 unique responses)}} \\
\midrule
\small   \textbf{Q.} what kind root stem does a dandelion have \par
  \textbf{A.} Fibrous roots and hollow stem. &
\small vitamin a, vitamin c, vitamin d and vitamin b complex, as well as zinc, iron and potassium.
(\texttt{snet}, \texttt{snet.ens}; $\times$ / $0.09$) 
\\
  \bottomrule
\end{tabular}

  \caption{\textbf{MS MARCO.} Human annotators rated answer correctness (\texttt{AnyCorrect}) and the automatic metric used is ROUGE-L (higher is better).}
  \end{subtable} \vspace{1em} \\
  \begin{subtable}{\textwidth}
  \centering
\begin{tabular}{p{0.42\textwidth} p{0.52\textwidth}}
\toprule
\textbf{Reference summary} & \textbf{System summary} \textbf{(System; \texttt{Edit} / VecSim)} \\
\midrule
\multicolumn{2}{l}{\textit{Examples where system \texttt{Edit} $< 0.3$ and VecSim $> 0.5$ ($53.9\%$ or 1078 of 2000 responses)}} \\
\midrule
\small Bhullar is set to sign a $\blacksquare$-day contract with the Kings. The $\blacksquare$-year-old will become the NBA's first player of Indian descent. Bhullar will be on the roster when the Kings host New Orleans Pelicans. &
\small \st{Bhullar and}\textbf{The} Kings are signing Bhullar to a $\blacksquare$-day contract. The $\blacksquare$-year-old will be on the roster on friday when David Wear's $\blacksquare$-season contract expires thursday. Bhullar is set to become the NBA's first player of Indian descent.
    (\texttt{ml}; $0.13$ / $0.82$)
    \\
\midrule
\multicolumn{2}{l}{\textit{Examples where system \texttt{Edit} $> 0.3$ and VecSim $> 0.5$ ($18.0\%$ or 360 of 2000 responses)}} \\
\midrule
\small The Direct Marketing Commission probing B2C Data and Data Bubble. Investigating whether they breached rules on the sale of private data. Chief commissioner described allegations made about firms as `serious'.  &
\small \st{$\blacksquare$ Data obtained by the Mail's marketing commission said it would probe both companies over claims that they had breached the rules on the sale of private data.} The FSA said it would probe both companies over claims they had breached the rules on the sale of private data. (\texttt{se2seq}; $1.00$ / $0.72$)
\\
\midrule
\multicolumn{2}{l}{\textit{Examples where system \texttt{Edit} $< 0.3$ and VecSim $< 0.5$ ($14.5\%$ or 290 of 2000 responses)}} \\
\midrule
\small Death toll rises to more than $\blacksquare$. Pemba Tamang, $\blacksquare$, shows no apparent signs of serious injury after rescue. Americans special forces helicopter $\blacksquare$, including $\blacksquare$ Americans, to safety. &
\small \st{Six of} \textbf{Despite} Nepal's tragedy, life triumphed in Kathmandu's hard-hit neighborhoods. Rescuers pulled an 15-year-old from the rubble of a multistory residential building. He was wearing a New York shirt and a blue neck brace. (\texttt{pointer}; $0.04$ / $0.27$)
\\
\midrule
\multicolumn{2}{l}{\textit{Examples where system \texttt{Edit} $> 0.3$ and VecSim $< 0.5$ ($13.6\%$ or 272 of 2000 responses)}} \\
\midrule
\small ``Mad Men's'' final seven episodes begin airing April $\blacksquare$. The show has never had high ratings but is considered one of the great TV series. It's unknown what will happen to characters, but we can always guess. &
\small `\st{This's} ``Mad Men'' is the end of a series of an era', \st{This} \textbf{he} says. Stores have created fashion lines inspired by the show.\st{``The Sopranos''. The in $\blacksquare$ the Kent State shootings in may $\blacksquare$ or Richard Nixon\'s $\blacksquare$ re-election..} (\texttt{ml+rl}; $0.95$ / $0.24$)\\
  \bottomrule
\end{tabular}

  \caption{\textbf{CNN/Daily Mail.} Human judgment scores used are post-edit distance (\texttt{Edit}) (lower is better) and the automatic metric used is sentence vector similarity with the reference (higher is better).}
  \end{subtable}
  \caption{\label{tab:examples}
    Examples highlighting the different modes in which the automatic metric and human judgments may agree or disagree.
    On the MS MARCO task, a majority of responses from systems were actually correct but poorly scored according to ROUGE-L.
    On the CNN/Daily Mail task, a significant number of examples which are scored highly by VecSim are poorly rated by humans, and likewise many examples scored poorly by VecSim are highly rated by humans.
  }
\end{table*}

It is well understood that current automatic metrics tend to correlate poorly with human judgment at the instance-level.
For example, \citet{novikova2017why} report correlations less than $0.3$ for a large suite of word-based and grammar-based evaluation methods on a generation task.
Similarly, \citet{liu2016evaluate} find correlations less than $0.35$ for automatic metrics on a dialog generation task in one domain, but find correlations with the same metric dropped significantly to less than $0.16$ when used in another domain. 
Still, somewhat surprisingly, several automatic metrics have been found to have high \textit{system-level} correlations~\citep{novikova2017why}.
What, then, are the implications of having a low instance-level correlation?  

As a case study, consider the task of open-response question answering:
  here, a system receives a human-generated question and must \textit{generate} an answer from some given context, e.g.\ a document or several webpages.
  We collected the responses of several systems on the MS MARCOv1 dataset~\citep{nguyen2016ms} and crowdsourced human evaluations of the system output
  (see \refsec{tasks} for details).

The instance-level correlation (\reffig{bias-msmarco-instance}) is only $\rho = 0.31$.
A closer look at the instance-level correlation reveals that
while ROUGE is able to correctly assign low scores to bad examples (lower left),
it is bad at judging good examples and often assigns them low ROUGE scores (lower right)---see \reftab{examples} for examples.
This observation agrees with a finding reported in \citet{novikova2017why} that automatic metrics correlate better with human judgments on bad examples than average or good examples.

Thus, as \reffig{bias-msmarco}(a) shows, we can improve low-scoring ROUGE examples without improving their human judgment ($\vartriangle$) and vice versa ($\triangleright$).
Indeed, \citet{conroy2008mind} report that summarization systems were optimized for ROUGE during the DUC challenge~\citep{dang2006overview}
until they were indistinguishable from the ROUGE scores of human-generated summaries, but the systems had hardly improved on human evaluation.
Hill-climbing on ROUGE can also lead to a system that does worse on human scores, e.g.\ in machine translation~\citep{wu2016google}.
Conversely, genuine quality improvements might not be reflected in improvements in ROUGE\@.
This bias also appears in pool-based evaluation for knowledge base population \citep{chaganty2017unbiased}.
Thus the problems with automatic metrics clearly motivate the need for human evaluation,
but can we still use the automatic metrics somehow to save costs?

\section{Statistical estimation for unbiased evaluation}
\label{sec:method}

We will now formalize the problem of combining human evaluation with an automatic metric.
Let $\sX$ be a set of inputs (e.g., articles),
and let $S$ be the \emph{system} (e.g.\ for summarization),
which takes $x \in \sX$ and returns output $S(x)$ (e.g.\ a summary).
Let $\sZ = \{ (x, S(x)) : x \in \sX \}$ be the set of system predictions.
Let $Y(z)$ be the random variable representing the human judgment according to some evaluation prompt (e.g.\ grammaticality or correctness),
and define $f(z) = \E[Y(z)]$ to be the (unknown) \emph{human metric} corresponding to averaging over an infinite number of human judgments.
Our goal is to estimate the average across all examples:
\begin{align}
\mu \eqdef \E_z[f(z)] = \inv{|\sZ|} \sum_{z \in \sZ} f(z)
\end{align}
with as few queries to $Y$ as possible.

Let $g$ be an automatic metric (e.g. ROUGE), which maps $z$ to a real number.
We assume evaluating $g(z)$ is free.
The central question is how to use $g$ in conjunction with calls to $Y$ to produce an unbiased estimate $\hat\mu$ (that is, $\E[\hat\mu] = \mu$).
In this section, we will construct a simple estimator based on control variates~\citep{ripley2009stochastic},
and prove that it is minimax optimal.

\subsection{Sample mean}

We warm up with the most basic unbiased estimate, the sample mean.
We sample $z^{(1)}, \dots, z^{(n)}$ independently with replacement from $\sZ$.
Then, we sample each human judgment $y^{(i)} = Y(z^{(i)})$ independently.\footnote{%
Note that this independence assumption isn't quite true in practice since we do not control who annotates our data.}
Define the estimator to be $\musimple = \frac{1}{n} \sum_{i=1}^n y^{(i)}$.
Note that $\musimple$ is unbiased ($\E[\musimple] = \mu$).

We can define $\sigma^2_f \eqdef \Var(f(z))$ as the variance of the human metric
and $\sigma^2_a \eqdef \E_z[\Var(Y(z))]$ as the variance of human judgment averaged over $\sZ$.
By the law of total variance, the variance of our estimator
is
\begin{align}
\label{eqn:varsimple}
\Var(\musimple) = \frac{1}{n} (\sigma^2_f + \sigma^2_a).
\end{align}

\subsection{Control variates estimator}

Now let us see how an automatic metric $g$ can reduce variance.
If there is no annotator variance ($\sigma^2_a = 0$) so that $Y(z) = f(z)$,
we should expect the variance of $f(z)-g(z)$ to be lower than the variance of
$f(z)$, assuming $g$ is correlated with $f$---see \reffig{variance_reduction} for an illustration.

The actual control variates estimator needs to
handle noisy $Y(z)$ (i.e.\ $\sigma^2_a > 0$) and
guard against a $g(z)$ with low correlation.
Let us standardize $g$ to have zero mean and unit variance, because we have
assumed it is free to evaluate.
As before, let $z^{(1)}, \dots, z^{(n)}$ be independent samples from $\sZ$ and
draw $y^{(i)} = Y(z^{(i)})$ independently as well.
We define the \emph{control variates estimator} as
\begin{align}
\mucontrol = \frac{1}{n} \sum_{i=1}^n y^{(i)} - \alpha g(z^{(i)}),
\end{align}
where
\begin{align}
  \alpha \eqdef \Cov(f(z),g(z)).
\end{align}
Intuitively, we have averaged over $y^{(i)}$ to handle the noise introduced by $Y(z)$, and scaled $g(z)$ to prevent an uncorrelated automatic metric from introducing too much noise.

An important quantity governing the quality of an automatic metric $g$
is the correlation between $f(z)$ and $g(z)$ (recall that $g$ has unit variance):
\begin{align}
\rho \eqdef \frac{\alpha}{\sigma_f}.  
\end{align}

\begin{figure}
\centering
  \includegraphics[width=\columnwidth]{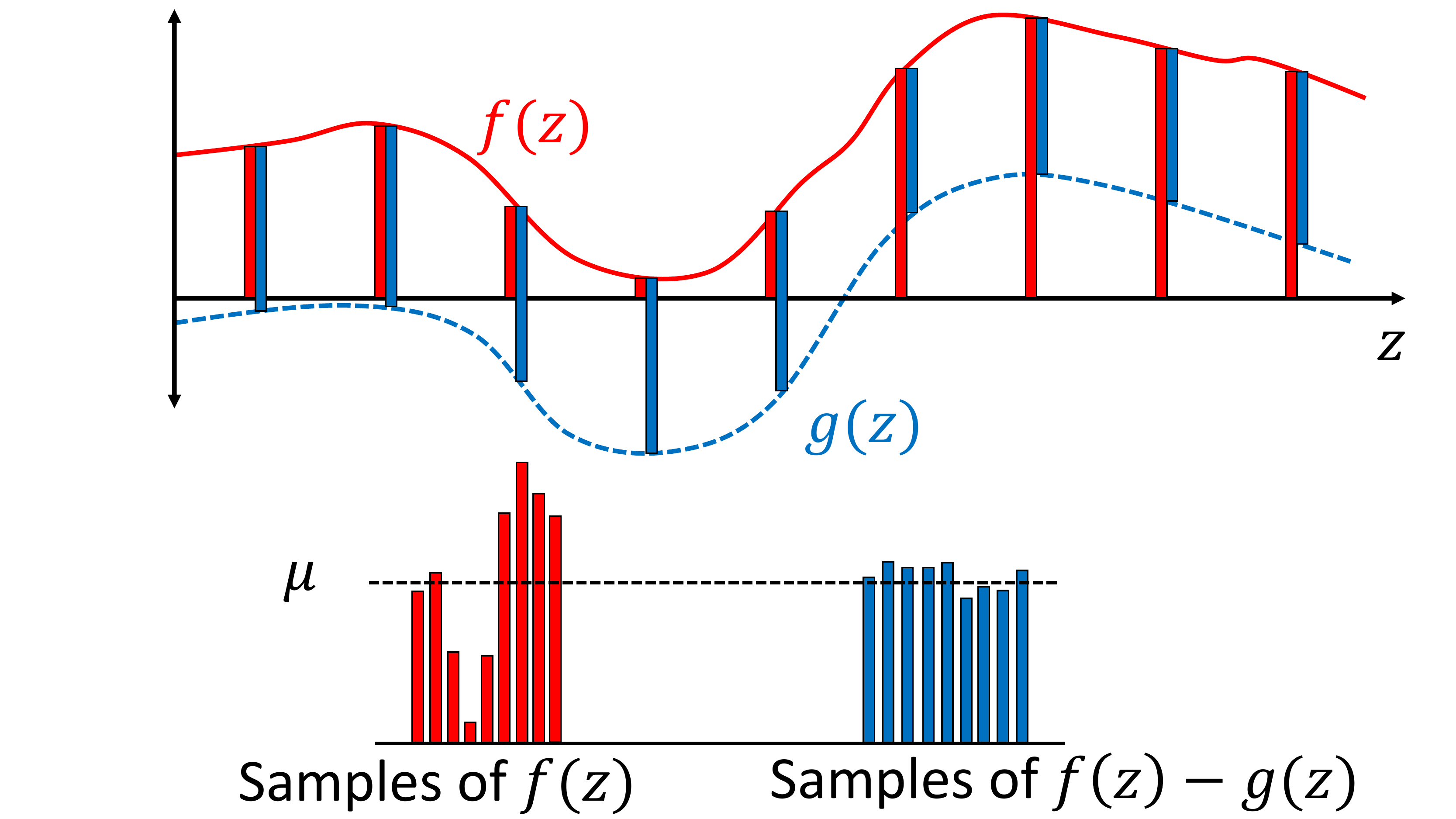}
  \caption{\label{fig:variance_reduction}
  The samples from $f(z)$ have a higher variance than the samples
  from $f(z)-g(z)$ but the same mean. This is the key idea behind using control variates to reduce variance.}
\end{figure}

We can show that among all distributions with
fixed $\sigma^2_f$, $\sigma^2_a$, and $\alpha$ (equivalently $\rho$), this estimator is minimax optimal, i.e.\ it has the least variance among all unbiased estimators:

\begin{theorem}
\label{thm:main}
Among all unbiased
  estimators that are functions of $y^{(i)}$ and $g(z^{(i)})$, and for all distributions with a given $\sigma^2_f$, $\sigma^2_a$, and $\alpha$,
\begin{align}
  \label{eqn:varcontrol}
  \Var(\mucontrol) = \frac{1}{n} (\sigma^2_f (1 - \rho^2) + \sigma^2_a),
\end{align}
and no other estimator has a lower worst-case variance.
\end{theorem}

Comparing the variances of the two estimators (\refeqns{varsimple}{varcontrol}),
we define the \emph{data efficiency} as the ratio of the variances:
\begin{align}
\DE \eqdef \frac{\Var(\musimple)}{\Var(\mucontrol)} = \frac{1 + \gamma}{1-\rho^2 + \gamma},
\end{align}
where $\gamma \eqdef \sigma^2_a / \sigma^2_f$ is the normalized annotator variance.
Data efficiency is the key quantity in this paper:
  it is the multiplicative reduction in the number of samples required
  when using the control variates estimator $\mucontrol$ versus the sample mean $\musimple$.
Figure~\ref{fig:savings} shows the inverse data efficiency contours as a function of the correlation $\rho$
and $\gamma$.

\begin{figure}
\centering
  \includegraphics[width=\columnwidth]{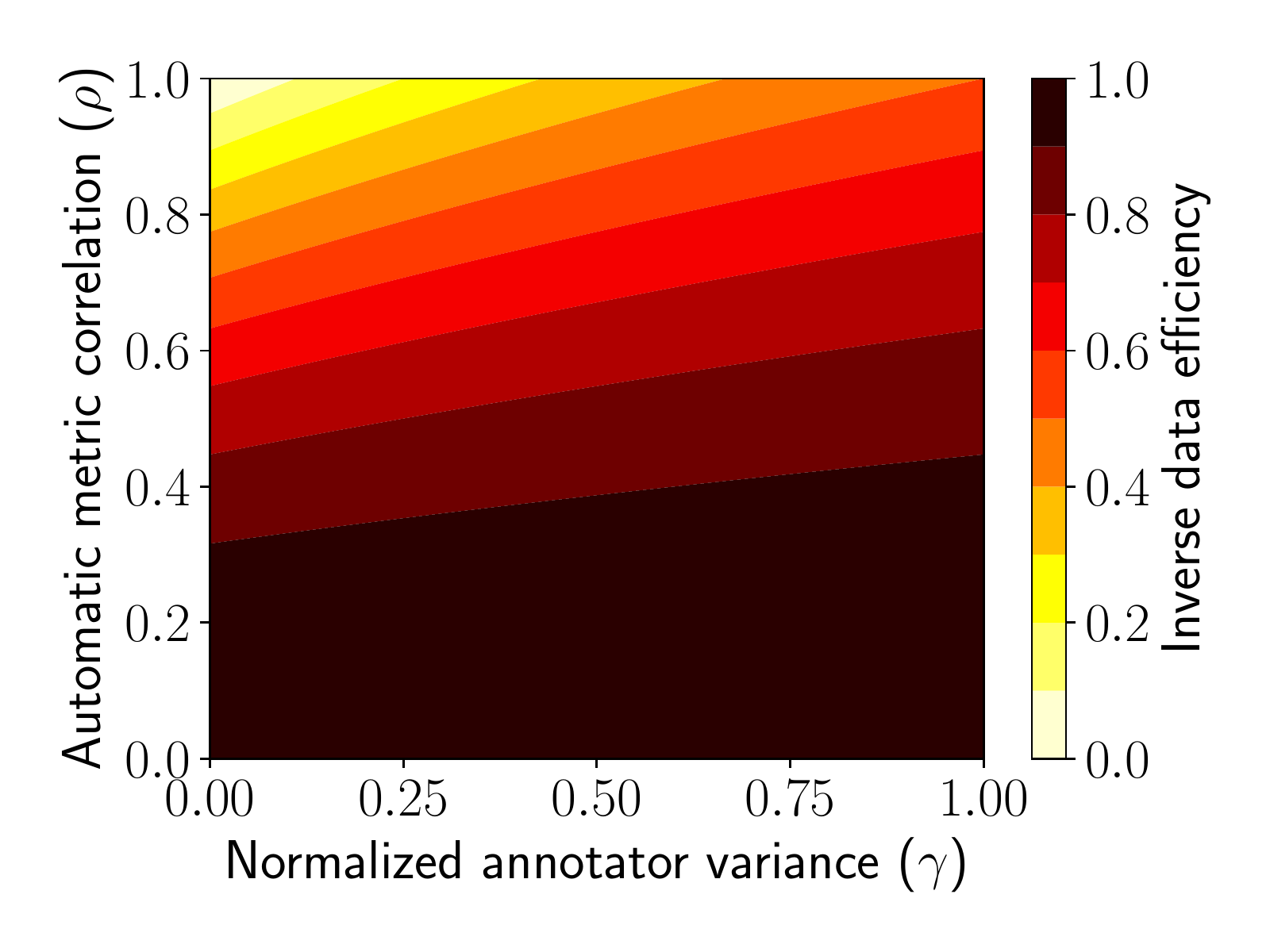}
  \caption{\label{fig:savings} Inverse data efficiency for various values of
  $\gamma$ and $\rho$.  We need both low $\gamma$ and high $\rho$ to obtain
  significant gains.
  }
\end{figure}

When there is no correlation between human and automatic metrics ($\rho = 0$),
the data efficiency is naturally $1$ (no gain).
In order to achieve a data efficiency of
$2$ (half the labeling cost), we need $|\rho| \geq \sqrt{2}/2 \approx 0.707$.
Interestingly, even for an automatic metric with perfect correlation ($\rho=1$),
the data efficiency is still capped by $\frac{1 + \gamma}{\gamma}$:
unless $\gamma \to 0$ the data efficiency cannot increase unboundedly.
Intuitively, even
if we knew that $\rho=1$, $f(z)$ would be undetermined up to a
constant additive shift and just estimating the shift would incur a variance of $\frac{1}{n} \sigma_a^2$.

\subsection{Using the control variates estimator}
The control variates estimator can be easily integrated into an existing evaluation:
we run human evaluation on a random sample of system outputs, automatic evaluation on all the system outputs, and plug in these results into \refalg{estimate}.

It is vital that we are able to evaluate the automatic metric on a significantly larger set of examples than those with human evaluations to reliably normalize $g(z)$:
without these additional examples, it be can shown that the optimal minimax estimator for $\mu$ is simply the naive estimate $\musimple$.
Intuitively, this is because estimating the mean of $g(z)$ incurs an equally large variance as estimating $\mu$.
In other words, $g(z)$ is only useful if we have additional information about $g$ beyond the samples $\{z^{(i)}\}$.

\refalg{estimate} shows the estimator.
In practice, we do not know $\alpha = \Cov(f(z),g(z))$, so we use a plug-in estimate $\hat{\alpha}$ in line 3 to compute the estimate $\widetilde{\mu}$ in line 4.
We note that estimating $\alpha$ from data does introduce a $O(1/n)$ bias,
but when compared to the standard deviation which decays as $\Theta(1/\sqrt{n})$, this bias quickly goes to $0$.

\begin{proposition}
\label{prop:added_bias}
The estimator $\widetilde{\mu}$ in \refalg{estimate} has $O(1/n)$ bias.
\end{proposition}

\begin{algorithm}
      \caption{\label{alg:estimate}Control variates estimator}
      \begin{algorithmic}[1]
   \STATE{} {\bfseries Input:} $n$ human evaluations $y^{(i)}$ on system outputs $z^{(i)}$, \textit{normalized} automatic metric $g$ 
   \STATE{} $\overline{y} = \frac{1}{n} \sum_i y^{(i)}$
   \STATE{} $\hat{\alpha} = \frac{1}{n} \sum_i (y^{(i)} - \overline{y}) g(z^{(i)})$
   \STATE{} $\widetilde{\mu} = \frac{1}{n} \sum_i y^{(i)} - \hat{\alpha} g(z^{(i)})$
   \STATE{} {\bfseries return} $\widetilde{\mu}$
\end{algorithmic}
\end{algorithm}

An additional question that arises when applying \refalg{estimate} is figuring out how many samples $n$ to use.
Given a target variance, the number of samples can be estimated using \refeqn{varcontrol} with conservative estimates of $\sigma^2_f$, $\sigma^2_a$ and $\rho$.
Alternatively, our estimator can be combined with a dynamic stopping rule~\citep{mnih2008empirical} to stop data collection once we reach a target confidence interval.

\subsection{Discussion of assumptions}
We will soon see that empirical instantiations of $\gamma$ and $\rho$ lead to rather underwhelming data efficiencies in practice.
In light of our optimality result, does this mean there is no hope for gains?
Let us probe our assumptions.
We assumed that the human judgments are uncorrelated across different system outputs;
it is possible that a more accurate model of human annotators (e.g.\ \citet{passonneau2014benefits}) could offer improvements.
Perhaps with additional information about $g(z)$ such as calibrated confidence estimates,
we would be able to sample more adaptively.
Of course the most direct routes to improvement involve increasing the correlation of $g$ with human judgments and reducing annotator variance,
which we will discuss more later.

\begin{figure*}[t]
  \centering
  \subcaptionbox{\label{fig:interfaces-edit}Interface to evaluate language quality on CNN/Daily Mail}[0.47\linewidth]{%
    \includegraphics[width=0.47\textwidth]{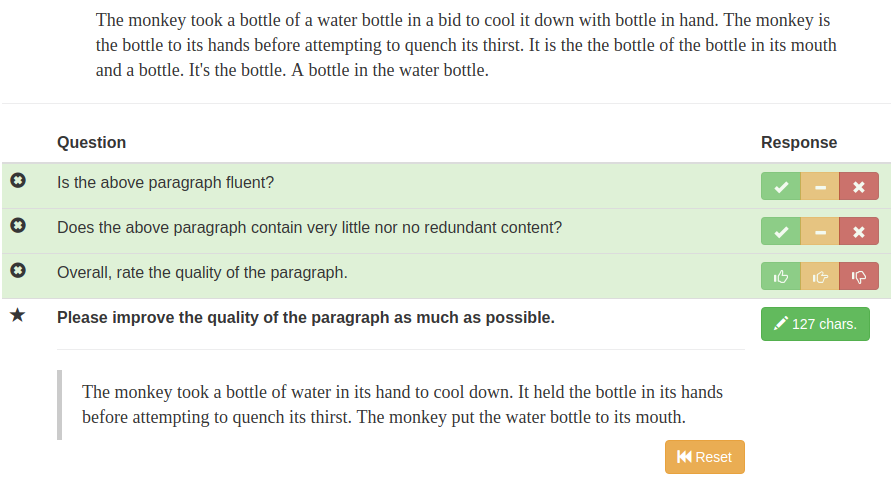}
  }\hfill
  \subcaptionbox{\label{fig:interfaces-qa}Interface to judge answer correctness on MS MARCO}[0.47\linewidth]{%
    \includegraphics[width=0.47\textwidth]{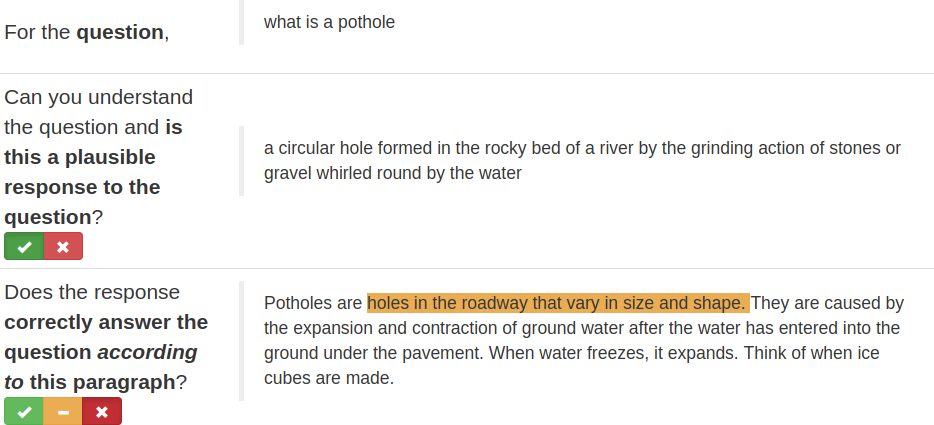}
  }
  \caption{\label{fig:tasks} Screenshots of the annotation interfaces we used to measure (a) summary language quality on CNN/Daily Mail and (b) answer correctness on MS MARCO tasks.
  }
\end{figure*}

\begin{table}[t]
  \centering
\begin{tabular}{l l r r r}
  \toprule
  Task & Eval. & $\sigma^2_a$ & $\sigma^2_f$ & $\gamma = \frac{\sigma^2_a}{\sigma^2_f}$ \\
  \midrule
  CDM & Fluency       & 0.32 & 0.26 & 1.23\\
  CDM & Redund.       & 0.26 & 0.43 & 0.61\\
  CDM & Overall       & 0.28 & 0.28 & 1.00\\
  \textbf{CDM} & \textbf{Edit}          & \textbf{0.07} & \textbf{0.18} & \textbf{0.36}\\
  \midrule
  MS MARCO & AnyCorr.          & 0.14 & 0.15 & 0.95\\
  MS MARCO & AvgCorr.          & 0.12 & 0.13 & 0.91\\
  \bottomrule
\end{tabular}
%
%


  \caption{\label{tab:dataset} A summary of the key statistics, human metric variance ($\sigma^2_f$) and annotator variance ($\sigma^2_a$) for different datasets, CNN/Daily Mail (CDM) and MS MARCO in our evaluation benchmark.
  We observe that the relative variance ($\gamma$) is fairly high for most evaluation prompts, upper bounding the data efficiency on these tasks.
  A notable exception is the \texttt{Edit} prompt wherein systems are compared on the number of post-edits required to improve their quality.
  }
\end{table}

\section{\label{sec:tasks} Tasks and datasets}

In order to compare different approaches to evaluating systems, we first collected human judgments for the output of several automatic summarization and open-response question answering systems using Amazon Mechanical Turk.
Details of instructions provided and quality assurance steps taken are provided in \refapp{interfaces} of the supplementary material.
In this section, we'll briefly describe how we collected this data.

\paragraph{Evaluating language quality in automatic summarization.}
In automatic summarization, systems must generate a short (on average two or three sentence) summary of an article: for our study, we chose articles from the CNN/Daily Mail (CDM) dataset~\citep{hermann2015read,nallapati2016abstractive} which come paired with reference summaries in the form of story highlights.
We focus on the \textit{language quality} of summaries and leave evaluating content selection to future work.

For each summary, we collected human judgments on a scale from 1--3 (\reffig{interfaces-edit}) for fluency, (lack of) redundancy, and overall quality of the summary using guidelines from the DUC summarization challenge~\citep{dang2006overview}.
As an alternate human metric, we also asked workers to post-edit the system's summary to improve its quality, similar to the post-editing step in MT evaluations~\citep{snover2006ter}.
Obtaining judgments costs about \$0.15 per summary and this cost rises to about \$0.40 per summary for post-editing.

We collected judgments on the summaries generated by the \texttt{seq2seq} and \texttt{pointer} models of \citet{see2017point}, the \texttt{ml} and \texttt{ml+rl} models of \citet{paulus2018deep}, and the reference summaries.\footnote{%
All system output was obtained from the original authors through private communication.} 
Before presenting the summaries to human annotators, we performed some minimal post-processing: we true-cased and de-tokenized the output of \texttt{seq2seq} and \texttt{pointer} using Stanford CoreNLP~\citep{manning2014stanford} and replaced ``unknown'' tokens in each system with a special symbol ($\blacksquare$).

\paragraph{Evaluating answer correctness.}
Next, we look at evaluating the correctness of system outputs in question answering using the MS MARCO question answering dataset~\citep{nguyen2016ms}.
Here, each system is provided with a question and up to 10 paragraphs of context.
The system generates open-response answers that do not need to be tied to a span in any paragraph.

We first ask annotators to judge if the output is even plausible for the question,
and if yes,
ask them identify if it is correct according to each context paragraph. 
We found that requiring annotators to highlight regions in the text that support their decision
substantially improved the quality of the output without increasing costs.
Annotations cost \$0.40 per system response.\footnote{%
  This cost could be significantly reduced if systems also specify which passage they used to generate the answer.
}

While our goal is to evaluate the correctness of the provided answer, we found that there are often answers which may be correct or incorrect depending on the context.
For example, the question ``what is a pothole'' is typically understood to refer to a hole in a roadway, but also refers to a geological feature (\reffig{interfaces-qa}).
This is reflected when annotators mark one context paragraph to support the given answer but mark another to contradict it.
We evaluated systems based on both the average correctness (AvgCorrect) of their answers across all paragraphs
as well as whether their answer is correct according to any paragraph (AnyCorrect).

We collected annotations on the systems generated by the \texttt{fastqa} and
\texttt{fastqa\_ext} from \citet{weissenborn2017fastqa} and the \texttt{snet} and \texttt{snet.ens}(emble) models from \citet{tan2018s}, along with reference answers.
The answers generated by the systems were used without any post-processing.
Surprisingly, we found that the correctness of the reference answers (according to the AnyCorrect metric) was only 73.5\%,
only 2\% above that of the leading system ($\texttt{snet.ens}$).
We manually inspected 30 reference answers which were annotated incorrectly and found that of those, 
about 95\% were indeed incorrect.
However, 62\% are actually answerable from some paragraph,
indicating that the real ceiling performance on this dataset is around 90\% and
that there is still room for improvement on this task.

\begin{figure*}[th]
  \centering
  \begin{subfigure}[b]{0.45\textwidth}
  \includegraphics[width=\textwidth]{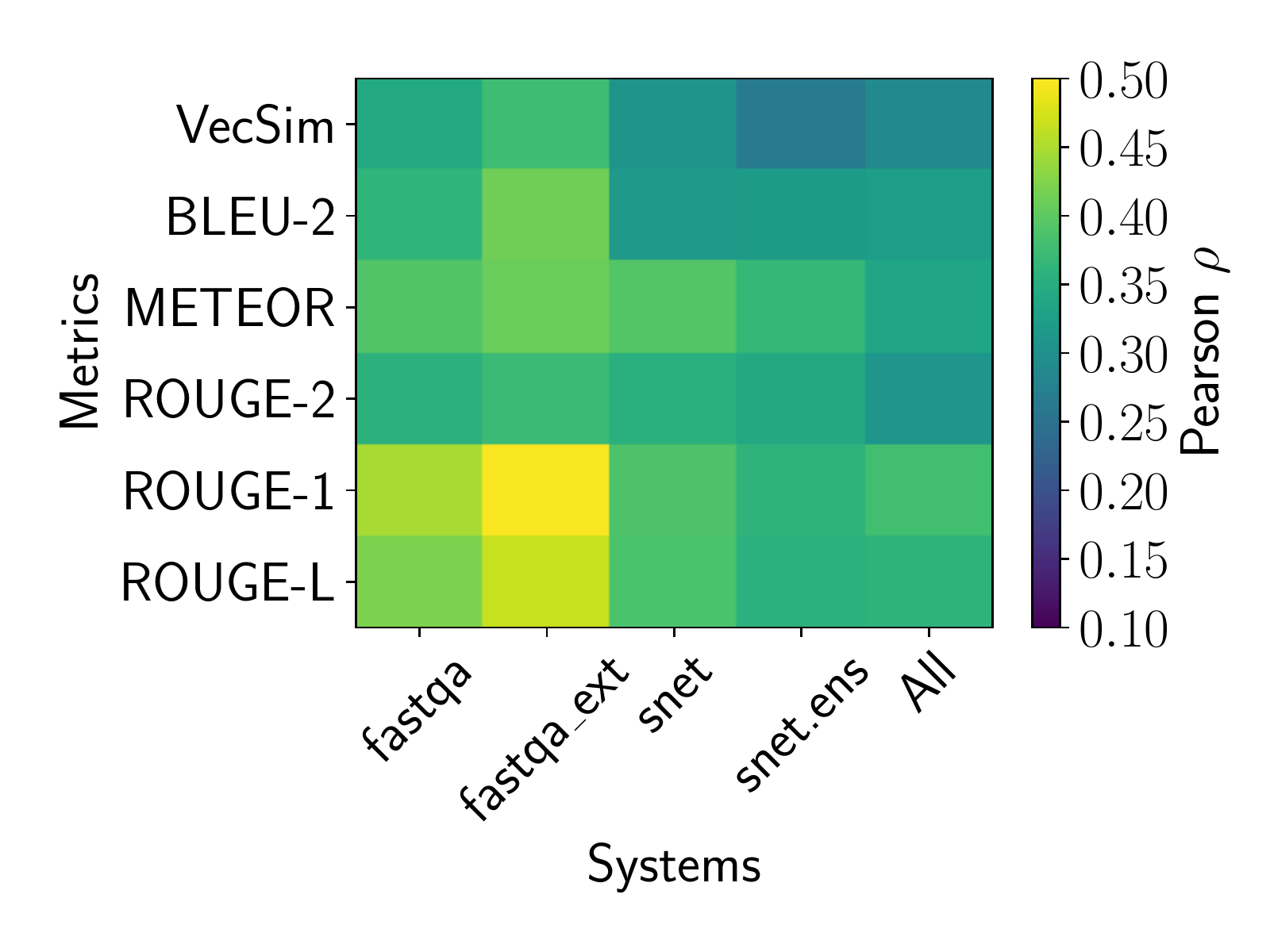}
  \caption{MS MARCO with the \texttt{AnyCorrect} prompt}
  \end{subfigure} \hfill
  \centering
  \begin{subfigure}[b]{0.45\textwidth}
  \includegraphics[width=\textwidth]{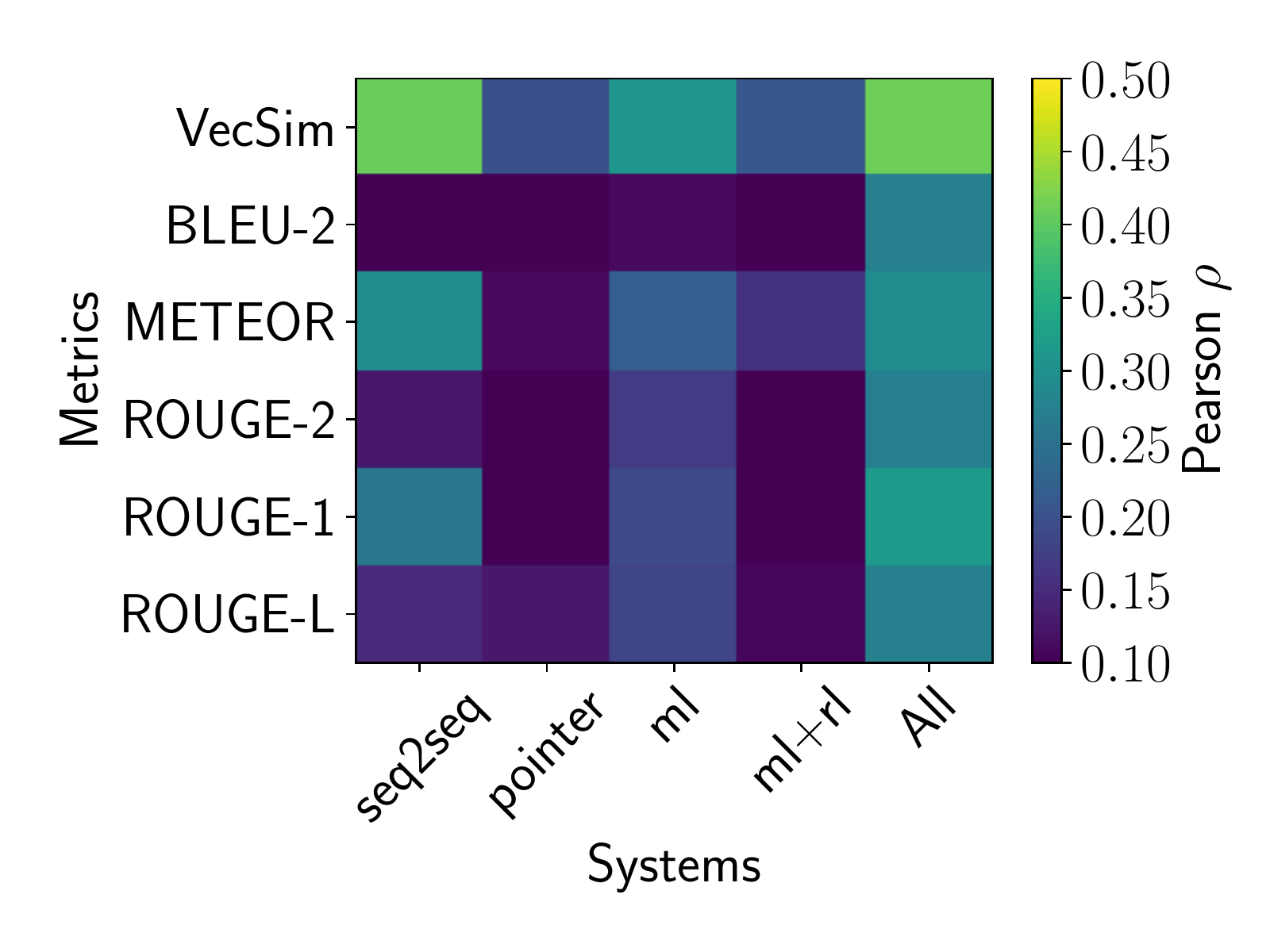}
  \caption{CNN/Daily Mail with the \texttt{Edit} prompt}
  \end{subfigure}
  \caption{\label{fig:correlation} Correlations of different automatic metrics on the MS MARCO and CNN/Daily Mail tasks.
  Certain systems are more correlated with certain automatic metrics than others, but overall the correlation is low to moderate for most systems and metrics.
  }
\end{figure*}

\section{\label{sec:evaluation}Experimental results}

We are now ready to evaluate the performance of our control variates estimator proposed in \refsec{method} using the datasets presented in \refsec{tasks}.
Recall that our primary quantity of interest is \textit{data efficiency}, the ratio of the number of human judgments required to estimate the overall human evaluation score for the control variates estimator versus the sample mean.
We'll briefly review the automatic metrics used in our evaluation before analyzing the results.

\paragraph{Automatic metrics.}
We consider the following frequently used automatic word-overlap based metrics in our work:
\textbf{BLEU}~\citep{papineni02bleu}, \textbf{ROUGE}~\citep{lin2004rouge} and \textbf{METEOR}~\citep{lavie2009meteor}.
Following \citet{novikova2017why} and \citet{liu2016evaluate}, we also compared a vector-based sentence-similarity using \texttt{sent2vec}~\citep{pagliardini2017unsupervised} to compare sentences (\textbf{VecSim}).
\reffig{correlation} shows how each of these metrics is correlated with human judgment for the systems being evaluated.
Unsurprisingly, the correlation varies considerably across systems, with token-based metrics correlating more strongly for systems that are more extractive in nature (\texttt{fastqa} and \texttt{fastqa\_ext}).

\begin{figure*}[th]
  \centering
  \begin{subfigure}[b]{0.32\textwidth}
  \includegraphics[width=\textwidth]{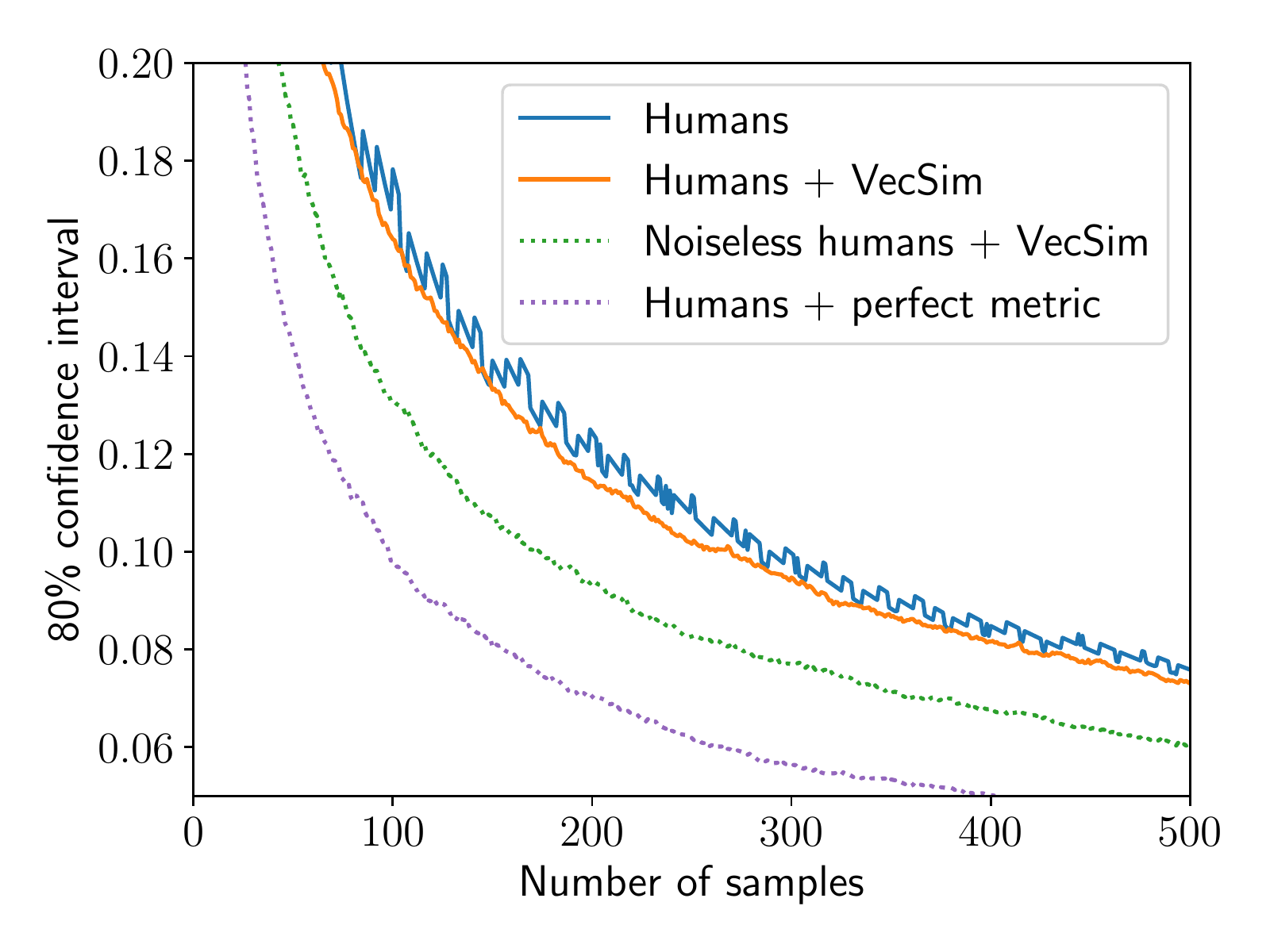}
    \caption{\label{fig:trajectory-a}\texttt{seq2seq} on CNN/Daily Mail using the \texttt{Overall}}
  \end{subfigure} 
  \hfill
  \begin{subfigure}[b]{0.32\textwidth}
  \includegraphics[width=\textwidth]{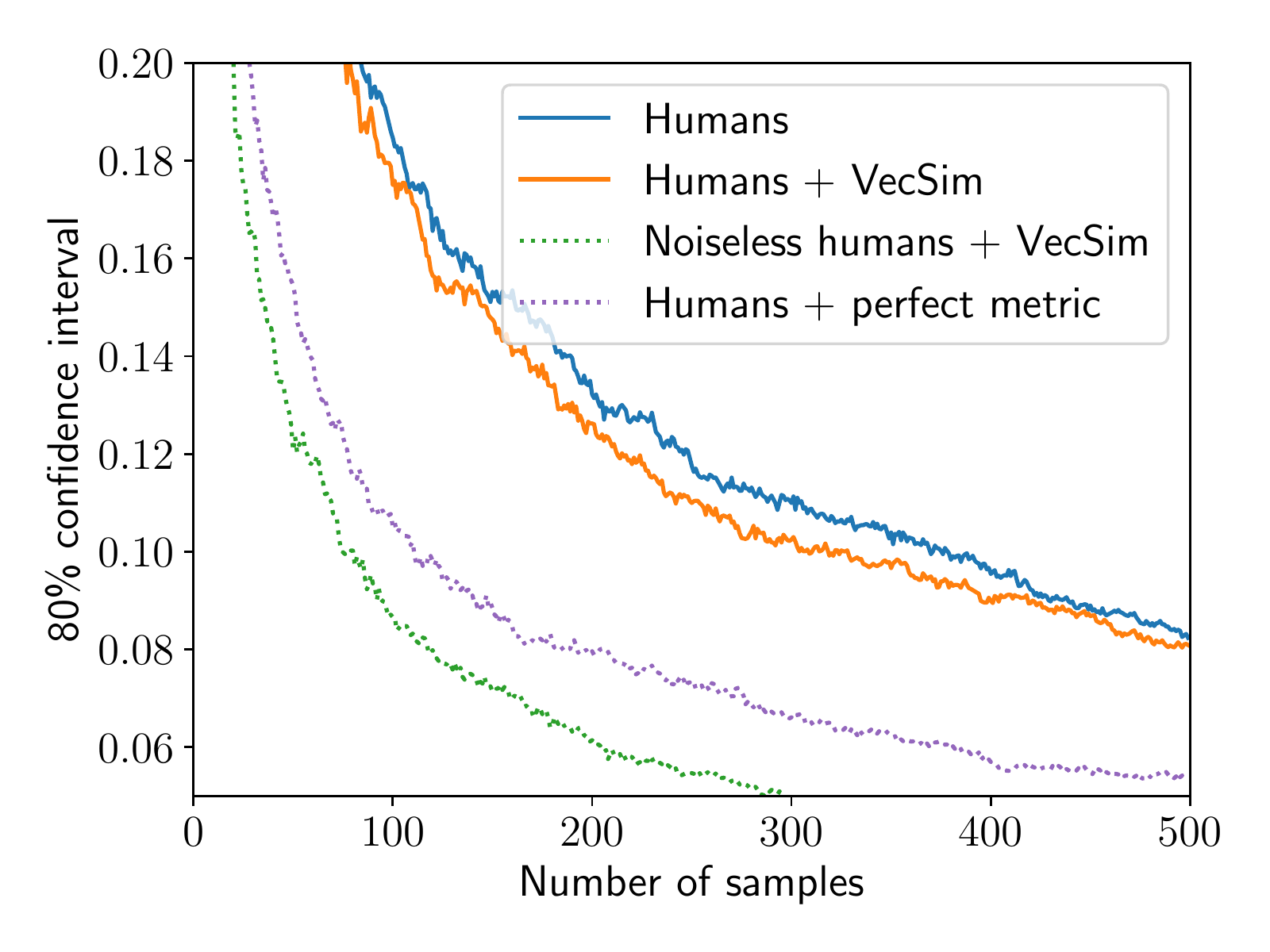}
  \caption{\label{fig:trajectory-b}\texttt{seq2seq} on CNN/Daily Mail using \texttt{Edit} }
  \end{subfigure}
  \hfill
  \begin{subfigure}[b]{0.32\textwidth}
  \includegraphics[width=\textwidth]{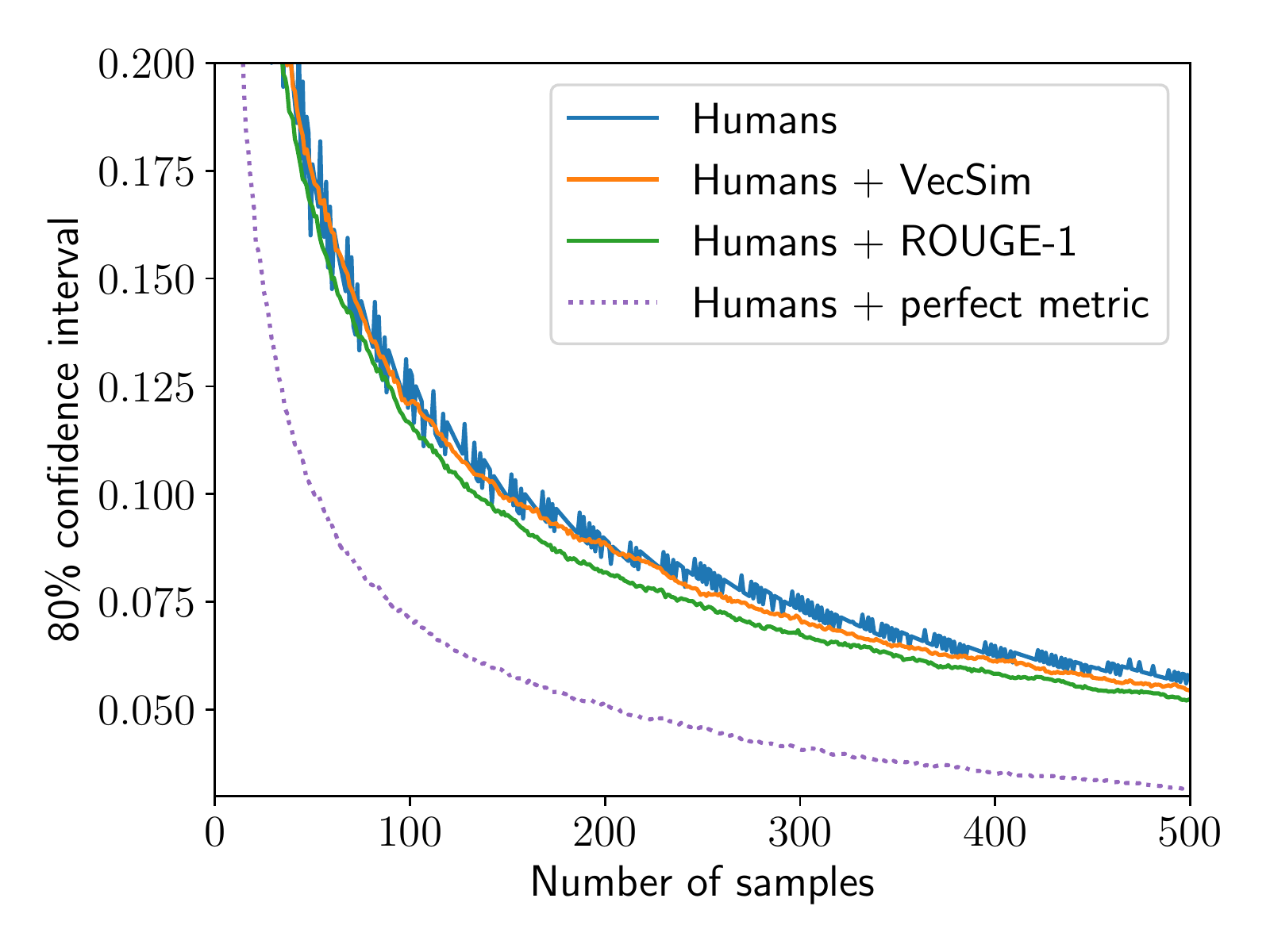}
  \caption{\label{fig:trajectory-c}\texttt{fastqa\_ext} on MS MARCO using \texttt{AnyCorrect}}
  \end{subfigure}
  \caption{\label{fig:trajectory} 80\% bootstrap confidence interval length as a function of the number of human judgments used when evaluating the indicated systems on their respective datasets and prompts.
  (a) We see a modest reduction in variance (and hence cost) relative to human evaluation by using the VecSim automatic metric with the proposed control variates estimator to estimate \texttt{Overall} scores on the CNN/Daily Mail task; the data efficiency (DE) is $1.06$.
  (b) By improving the evaluation prompt to use \texttt{Edit}s instead, it is possible to further reduce variance relative to humans (DE is $1.15$).
  (c) Another way to reduce variance relative to humans is to improve the automatic metric evaluation; here using ROUGE-1 instead of VecSim improves the DE from $1.03$ to $1.16$.
  }
\end{figure*}

\paragraph{Results.\footnote{%
  Extended results for other systems, metrics and prompts can be found at \url{https://bit.ly/price-of-debiasing/}.}
  }
In \refsec{method} we proved that the control variates estimator is not only unbiased but also has the least variance among other unbiased estimators.
\reffig{trajectory} plots the width of the 80\% confidence interval, estimated using bootstrap, measured as a function of the number of samples collected for different tasks and prompts.
As expected, the control variates estimator reduces the width of the confidence interval. 
We measure data efficiency by the averaging of the ratio of squared confidence intervals between the human baseline and control variates estimates.
We observe that the data efficiency depends on the task, prompt and system, ranging from about 1.08 (a 7\% cost reduction) to 1.15 (a 13\% cost reduction) using current automatic metrics.

As we showed in \refsec{method}, further gains are fundamentally limited by the quality of the evaluation prompts and automatic metrics.
Figures~\ref{fig:trajectory-a} and~\ref{fig:trajectory-b} show how improving the quality of the evaluation prompt from a Likert-scale prompt for quality (\texttt{Overall}) to using post-editing (\texttt{Edit}) noticeably decreases variance and hence allows better automatic metrics to increase data efficiency.
Likewise, \reffig{trajectory-c} shows how using a better automatic metric (ROUGE-L instead of VecSim) also reduces variance.

\reffig{trajectory} also shows the conjectured confidence intervals if we were able to eliminate noise in human judgments (noiseless humans) or have a automatic metric that correlated perfectly with average human judgment (perfect metric).
In particular, we use the mean of all (2--3) humans on each $z$ for the perfect $g(z)$ and use the mean of all humans on each $z$ for the ``noiseless'' $Y(z)$.

In both cases, we are able to significantly increase data efficiency (i.e.\ decrease estimator variance).
With zero annotator variance and using existing automatic metrics,
the data efficiency ranges from 1.42 to 1.69. With automatic metrics with perfect correlation and current variance of human judgments,
it ranges from 2.38 to 7.25.
Thus, we conclude that it is important not only to improve our automatic metrics but also the evaluation prompts we use during human evaluation. 

\section{\label{sec:setup} Related work}

In this work, we focus on using existing automatic metrics to decrease the cost of human evaluations.
There has been much work on improving the quality of automatic metrics.
In particular, there is interest in learning models~\citep{lowe2017towards,dusek2017referenceless} that are able to optimize for improved correlations with human judgment.
However, in our experience, we have found that these learned automatic metrics have trouble generalizing to different systems.
The framework we provide allows us to safely incorporate such models into evaluation, exploiting them when their correlation is high but also not introducing bias when it is low.

Our key technical tool is control variates, a standard statistical technique used to reduce the variance of Monte Carlo estimates~\citep{ripley2009stochastic}.
The technique has also been used in machine learning and reinforcement learning to lower variance estimates of gradients~\citep{greensmith2004variance, paisley2012variational, ranganath2014black}.
To the best of our knowledge, we are the first to apply this technique in the context of language evaluation.

Our work also highlights the importance of human evaluation.
\citet{chaganty2017unbiased} identified a similar problem of systematic bias in evaluation metrics in the setting of knowledge base population and also propose statistical estimators that relies on human evaluation to correct bias.
Unfortunately, their technique relies on having a structured output (relation triples) that are shared between systems and does not apply to evaluating natural language generation.
In a similar vein, \citet{chang2017affordable} dynamically collect human feedback to learn better dialog policies.

\section{Discussion}
\label{sec:discussion}

Prior work has shown that existing automatic metrics have poor instance-level correlation with mean human judgment and that they score many good quality responses poorly.
As a result, the evaluation is systematically biased against genuine system improvements that would lead to higher human evaluation scores but not improve automatic metrics.
In this paper, we have explored using an automatic metric to decrease the cost of human evaluation without introducing bias.
In practice, we find that with current automatic metrics and evaluation prompts data efficiencies are only 1.08--1.15 (7--13\% cost reduction).
Our theory shows that further improvements are only possible by improving the correlation of the automatic metric and reducing the annotator variance of the evaluation prompt.
As an example of how evaluation prompts could be improved, we found that using post-edits of summarizes decreased normalized annotator variance by a factor of three relative to using a Likert scale survey.
It should be noted that changing the evaluation prompt also changes the underlying ground truth $f(z)$: it is up to us to find a prompt that still captures the essence of what we want to measure.

Without making stronger assumptions, the control variates estimator we proposed outlines the limitations of unbiased estimation.
Where do we go from here?
Certainly, we can try to improve the automatic metric (which is potentially as difficult as solving the task) and brainstorming alternative ways of soliciting evaluation (which has been less explored).
Alternatively, we could give up on measuring absolute scores, and seek instead to find techniques stably rank methods and thus improve them.
As the NLP community tackles increasingly difficult tasks, human evaluation will only become more important.
We hope our work provides some clarity on to how to make it more cost effective.

\section*{Reproducibility}
All code, data, and experiments for this paper are available
on the CodaLab platform at \url{https://bit.ly/price-of-debiasing}.

\section*{Acknowledgments}
We are extremely grateful to the authors of the systems we evaluated for sharing their systems' output with us.
We also would like to thank 
  Urvashi Khandelwal and Peng Qi for feedback on an earlier draft of the paper,
  the crowdworkers on Amazon Mechanical Turk and TurkNation for their work and feedback during the data collection process,
  and the anonymous reviewers for their constructive feedback.

\bibliography{all}
\bibliographystyle{acl_natbib}

\appendix
\clearpage
\section{\label{sec:interfaces} Crowdsourcing data collection}

In this section, we provide details regarding our the design of our annotation interfaces and the quality control measures we took.

\subsection{Language quality evaluation.}
Each human annotator was shown a short summary that was generated by a system from an article in the CNN/Daily Mail dataset or provided as a reference for that article.
The annotators were then asked to (a) provide Likert scale ratings of the summary on multiple facets (fluency, redundancy and overall quality) and (b) perform post-edits to correct any errors (\reffig{lqual-interface}).

\paragraph{Interface design choices.}
We found that using a five-level Likert scale increased annotator variance as annotators relative to a three-level Likert scale.
Annotators were provided specific cues to calibrate their Likert ratings through a tutorial and were reminded of these cues through tooltips on the rating buttons (see \reffig{lqual-tutorial} for an example).
If the annotators rated a summary as lacking along any facet, they were then forced to perform post-edits to ``improve [its] quality as much as possible''.
We found that forcing annotators to provide post-edits on examples significantly decreased the annotator variance even on the Likert ratings.

\begin{figure*}
  \begin{subfigure}{\textwidth}
    \centering
    \includegraphics[width=0.8\textwidth]{figures/edit}
    \caption{\label{fig:lqual-interface}}
  \end{subfigure}
  \begin{subfigure}{\textwidth}
    \centering
    \includegraphics[width=0.8\textwidth]{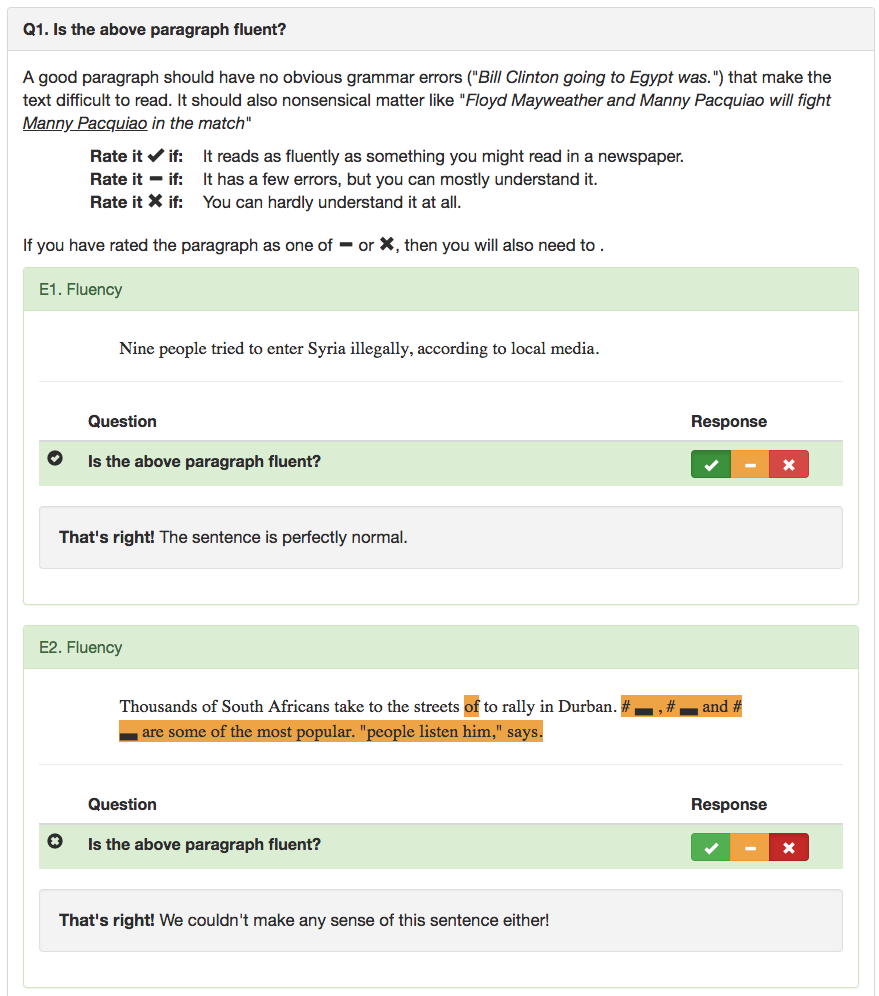}
    \caption{\label{fig:lqual-tutorial}}
  \end{subfigure}
  \caption{Screenshot of the (a) interface and (b) instructions used by crowdworkers for the language quality evaluation task on the CNN/Daily Mail dataset.}
\end{figure*}

\begin{figure*}
  \begin{subfigure}{\textwidth}
    \centering
    \includegraphics[width=0.8\textwidth]{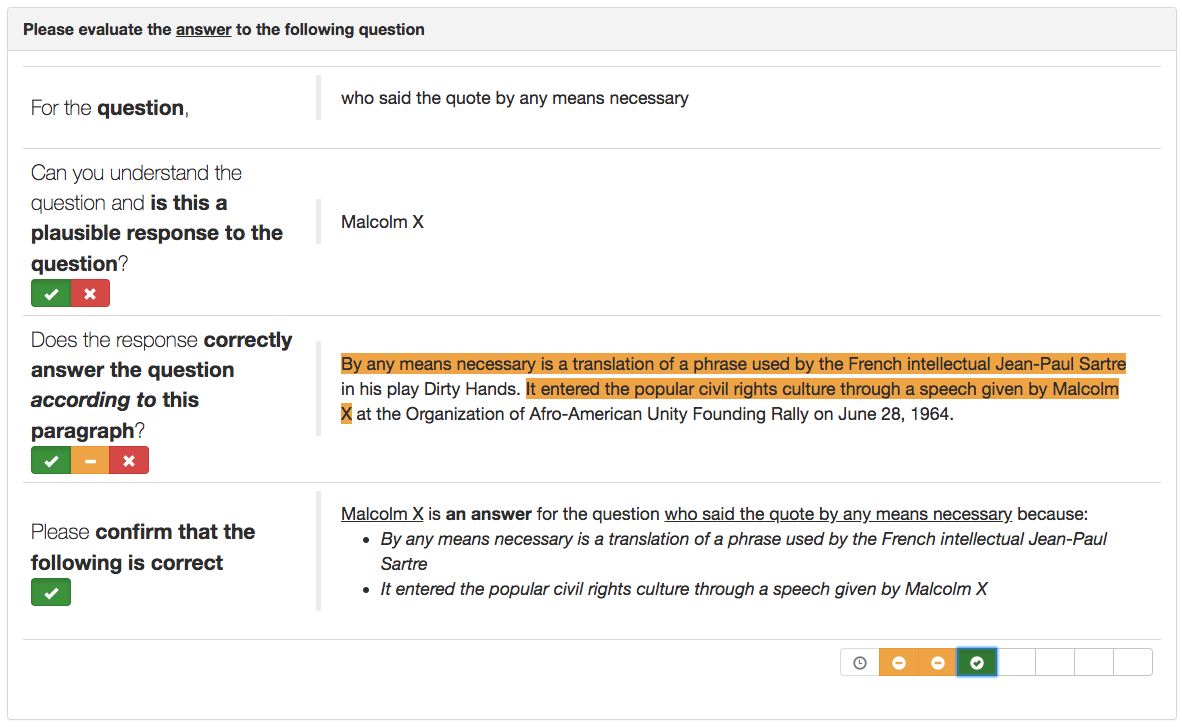}
    \caption{\label{fig:msmarco-interface}}
  \end{subfigure}
  \begin{subfigure}{\textwidth}
    \centering
    \includegraphics[width=0.8\textwidth]{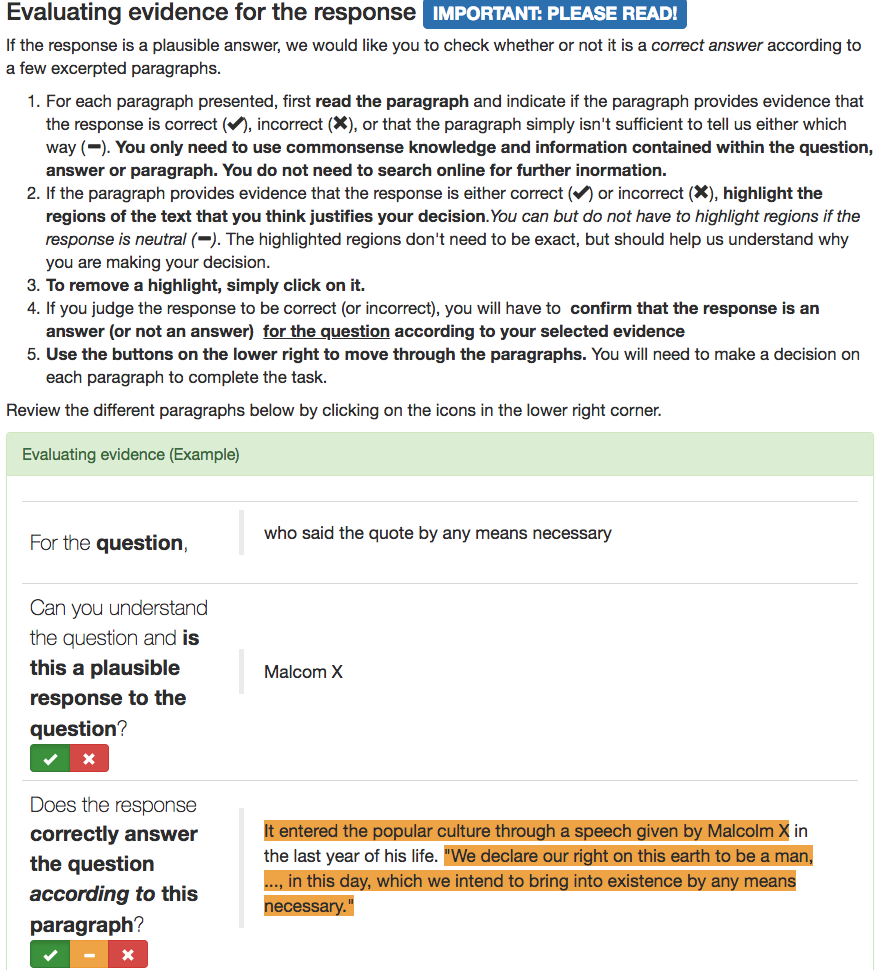}
    \caption{\label{fig:msmarco-tutorial}}
  \end{subfigure}
  \caption{Screenshot of the (a) interface and (b) instructions used by crowdworkers for the answer correctness evaluation task on the MS MARCO dataset.}
\end{figure*}

Following the recommendations of \citet{liu2016effective}, we forced annotators to complete an interactive tutorial containing 10 questions each before beginning the task (\reffig{lqual-tutorial}).
The tutorial provided guidelines and examples on how to rate each facet (fluency, redundancy and overall quality) and tested whether they were able to identify and correct language errors using the post-editing interface.
The tutorial took about 5--6 minutes to complete and annotators were paid a one-time bonus of \$0.75 on completion.

We initially included additional questions to assess focus, coherency and referential clarity adapted from the DUC evaluation guidelines~\citep{dang2006overview}, but found that annotators were unable to reliably identify these errors in the short summaries.
We also experimented with asking annotators to highlight language errors in the text to justify their ratings, but again found that annotators were unable to localize these errors reliably.

\paragraph{Quality control measures.}

We initially attempted to use attention-check examples for the Likert rating questions, but found that the ratings on these examples were themselves quite subjective and hence were not a reliable signal to reject work.
Instead, we found that requiring post-edits to summaries significantly reduced spam.
Additionally, we rejected annotators who took too little time to complete the task, had very low agreement rates on the Likert questions or had edits that were consistently shorter than 5 characters to prevent spam.

\subsection{Answer correctness evaluation.}
Each annotator was shown a question from the MS MARCO dataset and an answer that was generated by a system or provided as a reference answer from the dataset.
The annotators were then asked to (a) rate if the question made sense and the answer was plausibly correct and (b) asked to identify which paragraphs provided in the dataset justified the answer (\reffig{msmarco-interface}).

\paragraph{Interface design choices.}
We found that some of the questions in the MS MARCO dataset were extremely ambiguous (e.g.\ ``metatarsal what causes'') and some system responses were implausible (e.g\ ``monogenic bone diseases'', for the question ``what genes cause osteoporosis'').
In these cases, annotators expressed confusion if they were forced to judge if the response was correct or incorrect.
We resolved this confusion by first asking annotators if the question made sense and if system response was even plausible.

In early pilots, we found that annotators often rated a paragraph that correctly answered the question but was unrelated to the system response to be ``correct''.
We were able to resolve this problem by asking annotators to double-check their work (see the last question in \reffig{msmarco-interface} for an example).

Once again, we forced annotators to complete an interactive tutorial containing eight questions each before beginning the task (\reffig{msmarco-tutorial}).
The tutorial also took about 5--6 minutes to complete and annotators were paid a one-time bonus of \$0.75 on completion.

\paragraph{Quality control measures.}
We found that requiring annotators to provide justification spans significantly spam.
Additionally, we rejected annotators who took too little time to complete the task or had very low agreement rates on the answer correctness.

\onecolumn

\section{\label{sec:proofs}Proofs}

In this section, we provide proofs for the theorems stated in the main paper.

\subsection{Main Theorem}

In this section, we prove the main theorem (\refthm{main}) in the paper about the minimax optimal variance for an unbiased estimator. \refthm{main} will follow from the two following lemmas (Lemmas~\ref{lem:variance_calc} and~\ref{lem:mvue}). First, we show in \reflem{variance_calc} that for all distributions with fixed $\sigma^2_f$, $\sigma^2_a$ and $\rho$, the variance of $\mucontrol$ is constant and equal to: $\frac{1}{n}(\sigma_f^2(1-\rho^2)+\sigma_a^2)$.
Then we give an explicit distribution, a Gaussian distribution, where \emph{any} estimator yields at least this variance using the theory of sufficient statistics.
Together, these show that the max variance of any estimator is at least the max variance of $\mucontrol$.

As a reminder, the estimator is 

\begin{align}
\mucontrol = \frac{1}{n} \sum_i y^{(i)} - \alpha g(z^{(i)})
\end{align}

where $\alpha = \Cov(f(z),g(z))$.

\begin{lemma}
\label{lem:variance_calc}
The variance of $\mucontrol$ is always

\begin{align}
\frac{1}{n}(\sigma^2_f (1-\rho^2) + \sigma^2_a)
\end{align}

\end{lemma}
\begin{proof}
By the law of total variance, with respect to the draws of $z^{(i)}$,

\begin{align}
\Var(\mucontrol) = \mathbb{E}_{z^{(i)}}[\Var(\mucontrol | z^{(i)})] + \Var_{z^{(i)}} ( \mathbb{E}[\mucontrol | z^{(i)} ] ) 
\end{align}

We will evaluate each of the two terms on the right hand side. 

For the first term,

\begin{align}
\mathbb{E}_{z^{(i)}}[\Var(\mucontrol | z^{(i)})] = \mathbb{E}_{z^{(i)}} \left[ \Var \left( \frac{1}{n} \sum_i y^{(i)} | z^{(i)} \right) \right]
\end{align}
Because the human responses $Y(z^{(i)})$ are uncorrelated,

\begin{align}
\mathbb{E}_{z^{(i)}}[\Var(\mucontrol | z^{(i)})] &= \mathbb{E}_{z^{(i)}} \left[ \frac{1}{n^2} \sum_i \Var( Y(z^{(i)}) ) | z^{(i)} \right] \\
&= \frac{1}{n} \mathbb{E}_z [\Var(Y(z))] \\
&= \frac{1}{n} \sigma^2_a
\end{align}

For the second term,

\begin{align}
\Var_{z^{(i)}}( \mathbb{E}[\mucontrol | z^{(i)}]) = \Var_{z^{(i)}} \left( \frac{1}{n} \sum_i f(z^{(i)}) - \alpha g(z^{(i)}) \right)
\end{align}

Because the $z^{(i)}$ are sampled independently,

\begin{align}
\Var_{z^{(i)}}( \mathbb{E}[\mucontrol | z^{(i)}]) &= \frac{1}{n} \Var( f(z) - \alpha g(z)) \\
&= \frac{1}{n} [\Var(f(z)) - 2 \alpha \Cov(f(z),g(z)) + \alpha^2 \Var(g(z))]
\end{align}

Note that $\Var(f(z))=\sigma^2_f$, $\Cov(f(z),g(z))=\alpha$, and $\Var(g(z))=1$ (since it is normalized). Thus,

\begin{align}
\Var_{z^{(i)}}( \mathbb{E}[\mucontrol | z^{(i)}]) &= \frac{1}{n} [\sigma^2_f - 2 \alpha^2 + \alpha^2] \\
&=\frac{1}{n} [\sigma^2_f - \alpha^2]
\end{align}

Since the correlation $\rho = \frac{\alpha}{\sigma_f \sigma_g} = \frac{\alpha}{\sigma_f}$,

\begin{align}
\Var_{z^{(i)}}( \mathbb{E}[\mucontrol | z^{(i)}]) &= \frac{1}{n} [\sigma^2_f - \sigma^2_f \rho^2] \\
&=\frac{1}{n} \sigma^2_f (1-\rho^2)
\end{align}

Putting these two terms together, we find that,

\begin{align}
\Var(\mucontrol) &= \frac{1}{n} \sigma^2_a + \frac{1}{n} \sigma^2_f (1-\rho^2) \\
&= \frac{1}{n}(\sigma^2_f (1 - \rho^2) + \sigma^2_a)
\end{align}
\end{proof}

For the next lemma, we show that the worst-case variance for any estimator is at least that of $\mucontrol$. For this, we will define a simple Gaussian distribution and use the theory of sufficient statistics. We explicitly define a distribution over $f(z)$, $g(z)$, and $Y(Z) - f(z)$. In particular, we assume these are all Gaussian distributions with respective means, $\mu, 0, 0$,  and variances, $\sigma^2_f, 1, \sigma^2_a$. Additionally, we assume that $f(z)$ and $g(z)$ have covariance $\alpha$ but $Y(z) - f(z)$ is independent.

\begin{lemma}
\label{lem:mvue}
$\mucontrol$ is the minimal variance unbiased estimate (MVUE) for the Gaussian distribution above.
\end{lemma}
\begin{proof}
The proof is straightforward: we first show that $\mucontrol$ is a sufficient statistic using the Fisher-Neyman factorization theorem, and then we apply the Lehman-Scheffe theorem.

For ease of notation, define $g_i = g(z^{(i)})$ and $y_i = y^{(i)}$. For the purposes of statistics, only $\mu$ is a parameter; the other ``parameters'' are known constants. 
Note that the pdf of the observed variables $g_i$ and $y_i$ is,

\begin{align}
\prod_i c_1 \exp(-\frac{1}{2} 
\begin{bmatrix}
(y_i - \mu) \\
g_i \\
\end{bmatrix}^T
\begin{bmatrix}
\sigma^2_f + \sigma^2_a & \alpha \\
\alpha & 1 \\
\end{bmatrix}^{-1}
\begin{bmatrix}
(y_i - \mu) \\
g_i \\
\end{bmatrix})
\end{align}

\begin{align}
=c_2 \exp(-\frac{1}{2} \sum_i
\begin{bmatrix}
(y_i - \mu) \\
g_i \\
\end{bmatrix}^T
\begin{bmatrix}
\sigma^2_f + \sigma^2_a & \alpha \\
\alpha & 1 \\
\end{bmatrix}^{-1}
\begin{bmatrix}
(y_i - \mu) \\
g_i \\
\end{bmatrix})
\end{align}

Thus, with the Fisher-Neyman factorization theorem, it suffices to show that the exponetiated term $T$ decomposes as a sum of a function that only depends on the data and a function that only depends on $\mucontrol$ and $\mu$.

\begin{align}
  T &=
\sum_i
\begin{bmatrix}
(y_i - \mu) \\
g_i \\
\end{bmatrix}^T
\begin{bmatrix}
\sigma^2_f + \sigma^2_a & \alpha \\
\alpha & 1 \\
\end{bmatrix}^{-1}
\begin{bmatrix}
(y_i - \mu) \\
g_i \\
\end{bmatrix}
\end{align}

Letting $c_3$ be the inverse determinant (which is constant),

\begin{align}
T &= c_3 \sum_i
\begin{bmatrix}
(y_i - \mu) \\
g_i \\
\end{bmatrix}^T
\begin{bmatrix}
1 & -\alpha \\
-\alpha & \sigma^2_f + \sigma^2_a \\
\end{bmatrix}
\begin{bmatrix}
(y_i - \mu) \\
g_i \\
\end{bmatrix} \\
&= c_3 \left[ \sum_i (y_i - \mu)^2 - 2 \alpha \sum_i (y_i - \mu) g_i + (\sigma^2_f + \sigma^2_a) \sum_i g_i^2 \right] \\
&= c_3 \left[ \sum_i y_i^2 - 2 \mu \sum_i y_i + n \mu^2 - 2 \alpha \sum_i y_i g_i + 2 \alpha \mu \sum_i g_i  + (\sigma^2_f + \sigma^2_a) \sum_i g_i^2 \right] \\
&=  -2c_3 \mu \left[\sum_i y_i -  \alpha \sum_i g_i \right]  + c_3n\mu^2 +  c_3 \left[ \sum_i y_i^2 - 2 \alpha \sum_i y_i g_i + (\sigma^2_f + \sigma^2_a) \sum_i g_i^2 \right] \\
&=  -2 n c_3 \mu \mucontrol  +  c_3n\mu^2 + c_3\left[ \sum_i y_i^2 - 2 \alpha \sum_i y_i g_i + (\sigma^2_f + \sigma^2_a) \sum_i g_i^2 \right]
\end{align}

Thus, we see the decomposition into the function of only the data on the right and only $\mu$ and $\mucontrol$ on the left. Thus, $\mucontrol$ is a sufficient statistic.

Further, $\mucontrol$ is an unbiased estimate of $\mu$ since $\mathbb{E}[g_i]=0$ and $\mathbb{E}[y_i] = \mu$. 

Further, since $\mucontrol$ is normally distributed with mean dependent on $\mu$, it is complete. 

Thus, by the Lehmann-Scheffe theorem, $\mucontrol$ is the minimal variance unbiased estimate (MVUE).

\end{proof}

\begin{thm} [\ref{thm:main}]
Among all unbiased
  estimators that are functions of $y^{(i)}$ and $g(z^{(i)})$, and for all distributions with a given $\sigma^2_f$, $\sigma^2_a$, and $\alpha$,
\begin{align}
  \Var(\mucontrol) = \frac{1}{n} (\sigma^2_f (1 - \rho^2) + \sigma^2_a),
\end{align}
and no other estimator has a lower worst-case variance.
\end{thm}
\begin{proof}
From \reflem{variance_calc} we have that the max variance of $\mucontrol$ over all distributions with fixed variances, is exactly,

\begin{align}
\frac{1}{n} (\sigma^2_f (1 - \rho^2) + \sigma^2_a)
\end{align}

Further, from \reflem{mvue}, we know that $\mucontrol$ is the MVUE for a particular class of distributions, thus, any estimator has a larger max variance over all distributions. 

Combining these two facts, we get that the minimax variance is the variance of $\mucontrol$.
\end{proof}

\subsection{Added Bias}

\begin{prop}[\ref{prop:added_bias}]
The estimator in Algorithm \ref{alg:estimate} has $O(1/n)$ bias.
\end{prop}
\begin{proof}

The bias $B$ is

\begin{align}
B &= \left| \mathbb{E}[\widetilde{\mu}] - \mu \right| \\
&= \left| \mathbb{E}[\frac{1}{n} \sum_i y^{(i)} - \hat{\alpha} g(z^{(i)})] - \mu \right|
\end{align}

Since $\mathbb{E}[y^{(i)}] = \mu$,

\begin{align}
B&= \left| \mu - \frac{1}{n} \sum_i \mathbb{E}[\hat{\alpha} g(z^{(i)})] - \mu \right| \\
&=\left| \frac{1}{n} \sum_i \mathbb{E}[\hat{\alpha} g(z^{(i)})] \right| \\
&=\left| \frac{1}{n^2} \sum_{i,j} \mathbb{E}[(y^{(j)} - \overline{y}) g(z^{(j)}) g(z^{(i)})]\right| \\
&=\left| \frac{1}{n^2} \sum_{i,j} \mathbb{E}[y^{(j)} g(z^{(j)}) g(z^{(i)})] - \frac{1}{n^3} \sum_{i,j,k} \mathbb{E}[y^{(k)} g(z^{(j)}) g(z^{(i)})] \right| 
\end{align}

Because $Y(z)$ is independent and has mean $f(z)$,

\begin{align}
B=\left| \frac{1}{n^2} \sum_{i,j} \mathbb{E}[f(z^{(j)}) g(z^{(j)}) g(z^{(i)})] - \frac{1}{n^3} \sum_{i,j,k} \mathbb{E}[f(z^{(k)}) g(z^{(j)}) g(z^{(i)})] \right| 
\end{align}

Because $g(z)$ is mean zero and the $z^{(i)}$ are drawn independently,

\begin{align}
B &= \left| \frac{1}{n^2} \sum_{i} \mathbb{E}[f(z^{(i)}) g(z^{(i)})^2] - \frac{1}{n^3} \sum_{i,k} \mathbb{E}[f(z^{(k)}) g(z^{(i)})^2] \right| \\
&= \left| \frac{1}{n^2} \sum_{i} O(1) - \frac{1}{n^3} \sum_{i,k} O(1) \right| \\
&= \left| \frac{1}{n^2} O(n) - \frac{1}{n^3} O(n^2) \right| \\
&= \left| O \left(\frac{1}{n} \right) - O \left(\frac{1}{n} \right) \right| \\
&= O \left(\frac{1}{n} \right) 
\end{align}

\end{proof}

\end{document}